\documentclass[11pt]{article}
\pdfoutput=1

\usepackage{ifthen}
\newboolean{usenatbib}
\setboolean{usenatbib}{false}

\ifthenelse{\boolean{usenatbib}}{
  \usepackage[numbers,sort]{natbib}
}{
  \usepackage[style=alphabetic,maxalphanames=12,maxcitenames=2,maxbibnames=99,giveninits=false,doi=false,url=false,natbib]{biblatex}
  \addbibresource{refs.bib}

  \DeclareDelimFormat[bib]{nametitledelim}{\addspace}
}

\usepackage{bm}
\usepackage{fullpage}
\usepackage[english]{babel}
\usepackage[utf8]{inputenc}
\usepackage[T1]{fontenc}
\usepackage{ae} \usepackage{aecompl}
\usepackage{enumitem}
\setlist{nosep}
\usepackage{microtype}
\usepackage{booktabs}
\usepackage{hhline}
\usepackage{makecell}
\usepackage{etoolbox}
\apptocmd{\sloppy}{\hbadness 10000\relax}{}{}

\usepackage{amsmath,amsthm,amssymb,mathtools,thmtools}
\usepackage{graphicx}
\usepackage[textsize=tiny]{todonotes}
\usepackage[colorlinks=true, allcolors=blue]{hyperref}
\usepackage{algorithm,algorithmicx,algpseudocode}
\usepackage{xfrac}
\usepackage[normalem]{ulem}

\hypersetup{colorlinks,
            breaklinks=true,
            linkcolor=blue,
            citecolor=blue,
            urlcolor=magenta,
            linktocpage,
            plainpages=false,
            bookmarks=false}

\usepackage{notations}
\usepackage{comment}
\usepackage{tikz,pgfplots,pgfplotstable}
\usepgfplotslibrary{units}
\usepackage{tabularx} 
\usepackage{subfig}
\usepackage{array}
\usepgfplotslibrary{groupplots}

\title{Mixtures of Gaussians are Privately Learnable \\ with a Polynomial Number of Samples}

\author{}

\numberwithin{equation}{section}

\begin{document}

\author{
    Mohammad Afzali\thanks{McMaster University, \texttt{afzalikm@mcmaster.ca}}
    \and 
    Hassan Ashtiani\thanks{McMaster University, \texttt{zokaeiam@mcmaster.ca}. Hassan Ashtiani is also a faculty affiliate at Vector Institute and supported by an NSERC Discovery Grant.}
     \and 
    Christopher Liaw\thanks{Google, \texttt{cvliaw@google.com}.}
}

\maketitle

\begin{abstract}%
We study the problem of estimating mixtures of Gaussians under the constraint of differential privacy (DP). Our main result is that $\text{poly}(k,d,1/\alpha,1/\eps,\log(1/\delta))$ samples are sufficient to estimate a mixture of $k$ Gaussians in $\bR^d$ up to total variation distance $\alpha$ while satisfying $(\eps, \delta)$-DP. This is the first finite sample complexity upper bound for the problem that does not make any structural assumptions on the GMMs.

To solve the problem, we devise a new framework which may be useful for other tasks. On a high level, we show that if a class of distributions (such as Gaussians) is \textit{(1)} list decodable and \textit{(2)} admits a ``locally small'' cover~\citep{bun2021private} with respect to total variation distance, then the class of \emph{its mixtures} is privately learnable. The proof circumvents a known barrier indicating that, unlike Gaussians, GMMs do not admit a locally small cover~\citep{aden2021privately}. 

\end{abstract}

\begin{section}{Introduction}
        Density estimation---also known as distribution learning---is the fundamental task of estimating a distribution given samples generated from it. In the most commonly studied setting, the data points are assumed to be generated independently from an unknown distribution $f$ and the goal is to find a distribution $\hat{f}$ that is close to $f$ with respect to the total variation (TV) distance. Assuming that $f$ belongs to (or is close to a member of) a class of distributions $\cF$, an important question is to characterize the number of samples that is required to guarantee that with high probability $\hat{f}$ is close to $f$ in TV distance.

        There is a large body of work on characterizing the optimal sample complexity (or the related minimax error rate) of learning various classes of distributions (see \citet{diakonikolas2016learning,devroye2001combinatorial, ashtiani2018some} for an overview). Nevertheless, determining the sample complexity (or even learnability) of a general class of distributions remains an important open problem (e.g., Open Problem 15.1 in~\citet{diakonikolas2016learning}). 

        We study the problem of density estimation under the constraint of differential privacy \citep{dwork2006calibrating}.            
        Intuitively, differential privacy (DP) requires that changing a single individual's data does not identifiably alter the outcome of the estimation. 
        There are various formulations of DP.            
        The original pure DP ($\eps$-DP) formulation can be somewhat restrictive, as learning some simple classes of distributions (such as univariate Gaussians with unbounded mean) is impossible in this model. An alternative formulation is approximate DP \citep{dwork2006our}, which is also known as $(\eps, \delta)$-DP. Interestingly, we are not aware of any class of distributions that is learnable in the non-private (agnostic\footnote{In the agnostic setting, we do not assume the true distribution belongs to the class that we are considering. The goal is then finding a distribution in the class that is (approximately) the closest to the true distribution. The previous arxiv version of this paper did not mention ``agnostic'' for the conjecture. However, we believe the conjecture is more likely to be true for the agnostic setting, given the recent developments on connections between robustness and privacy~\cite{asi2023robustness, hopkins2023robustness}.}) setting but not learnable in the $(\eps, \delta)$-DP setting. 
        This is in sharp contrast with other areas of learning such as classification\footnote{For example, initial segments over the real line are learnable in the non-private setting but not in the $(\eps, \delta)$-DP model; see~\cite{bun2015differentially, alon2019private,kaplan2020privately,bun2020equivalence, cohen2023optimal}.}.            
        In fact, we conjecture that every class of distributions that is learnable in the non-private (agnostic) setting is also learnable in the $(\eps, \delta)$-DP setting.

        Nevertheless, we are still far from resolving the above conjecture. Yet, we know that some important classes of distributions are learnable in $(\eps, \delta)$-DP setting. For example, finite hypothesis classes are learnable even in the pure DP setting~\citep{bun2021private,aden2021sample}. Earlier work of \citet{karwa2018finite} shows that unbounded univariate Gaussians are also learnable in this setting. 
        More generally, high-dimensional Gaussians with unbounded parameters are private learnable too~\citep{aden2021sample}, even with a polynomial time algorithm \citep{kamath2022private,kothari2022private,ashtiani2022private,alabi2023privately,hopkins2023robustness}. A next natural step is studying a richer class of distributions such as Gaussian Mixture Models (GMMs). 
        
        In the non-private setting, GMMs with $k$ components in $d$ dimensions are known to be learnable (with error/TV-distance at most $\alpha$) with a polynomial number of samples (in terms of $d,k,1/\alpha$). Perhaps surprisingly, nearly tight sample complexity bound of $\tilde{O}(d^2k/\alpha^2)$ was proved relatively recently using distribution sample compression schemes~\citep{ashtiani2018nearly,ashtiani2020near}. But are GMMs learnable in the $(\eps, \delta)$-DP setting with a polynomial number of samples?

        Indeed, \citet{aden2021privately} have shown that univariate GMMs are learnable in the $(\eps, \delta)$-DP setting. Namely, they use stability-based histograms~\citep{bun2016simultaneous} in the spirit of \citet{karwa2018finite} to come up with a set of candidate parameters for the mixture components, and then choose between these candidates using private hypothesis selection~\citep{bun2021private,aden2021sample}. While they generalize this idea to learning axis-aligned GMMs, their approach does not work for GMMs with general covariance matrices\footnote{In fact, it is not clear if it is possible to extend the histogram-based approach to handle arbitrary covariance matrices \emph{even for learning a single} (high-dimensional) Gaussian.}.
        
        Another idea to resolve the learnability of GMMs is to extend the result of~\citet{aden2021sample} for high-dimensional Gaussians. In particular, they show that Gaussians admit a ``locally small cover'' with respect to the total variation (TV) distance and therefore the class of Gaussians can be learned privately using the approach of \citet{bun2021private}. However, as \citet{aden2021privately} demonstrated, GMMs do not admit such a locally small cover with respect to TV distance.

        An alternative approach for private learning of GMMs would be using a sample-and-aggregate framework such as those proposed by~\citet{ashtiani2022private,tsfadia2022friendlycore}. In particular, \citet{ashtiani2022private} show how one can privately learn Gaussians by aggregating the outcomes of multiple non-private Gaussian estimators and then outputting a noisy version of those parameters. In fact, this is the basis of the work by~\citet{arbas2023polynomial} who showed how to reduce the problem of private \emph{parameter estimation} for GMMs into its non-private counterpart. However, while this reduction is (computationally and statistically) efficient, the non-private version of the problem itself requires an (unavoidable) exponential number of samples with respect to the number of components~\citep{moitra2010settling}.

        Can we avoid the (above mentioned) exponential dependence on $k$ if we opt for (private) density estimation rather than (private) parameter estimation? We know this is possible in the non-private setting~\citep{ashtiani2018nearly,ashtiani2020near, ashtiani2018sample} or when we have access to some ``public data''~\citep{ben2023private}. One idea is to use a sample-and-aggregate approach based on a non-private \emph{density estimator} for GMMs. This turns out to be problematic as GMMs are not uniquely parameterized: two GMMs may be close to each other in terms of total variation distance but with a completely different set of parameters. 
        Thus, it is challenging to use non-private algorithms for learning GMMs as a blackbox since one cannot guarantee that the outputs of these algorithms are ``stable''.

        In this paper, we bypass the above barriers and show that GMMs are privately learnable with a polynomial number of samples (in terms of $d,k,1/\alpha,1/\eps, \log (1/\delta)$). We do this via a fairly general reduction that shows if a class (such as Gaussians) admits a ``locally small cover'' and is ``list decodable'', then the class of \emph{its mixtures} is privately learnable. We find it useful to first define some notations in Section~\ref{sec:prelim} before stating our main results in Section~\ref{sec:results}. Given the subtleties in the proofs, we given an overview of our technical contributions in Section~\ref{sec:techniques} before delving into the core technical sections.

\end{section}

\begin{section}{Preliminaries}\label{sec:prelim}
For a set $\cF$, define $\cF^k= \cF  \times \dots \times \cF$ ($k$ times), and $\cF^*=\bigcup_{k=1}^{\infty} \cF^k$. We use $[k]$ to denote the set $\{1,2,\dots,k\}$. We use $S^d$ to denote the positive-definite cone in $\bR^{d\times d}$. Moreover, for two absolutely continuous densities $f_1(x),f_2(x)$ on $\bR^d$, the total variation (TV) distance is defined as $\dtv(f_1,f_2)=\frac{1}{2}\int_{\bR^d}|f_1(x)-f_2(x)|\, \dd x$. For a matrix $A$, let $\|A\|_F = \sqrt{\Tr(A^TA)}$ be the Frobenius norm and $\|A\|_2$ be the induced $\ell_2$ (spectral) norm. In this paper, if $\kappa$ is a metric on $\cF$, $f \in \cF$, and $Y \subseteq \cF$ then we define $\kappa(x, Y) = \inf_{y \in Y} \kappa(f, y)$.

\begin{definition} [$\kappa$-ball] \label{def:ball}
    Consider a metric space $(\cF,\kappa)$. For a fixed  $f \in \cF$, define $B_{\kappa}(r,f,\cF) \coloneqq \{f' \in \cF\,:\,\kappa(f',f) \leq r\}$ to be the $\kappa$-ball of radius r around $f$.
\end{definition}

\begin{definition} [$\alpha$-cover]
    A set $C_{\alpha} \subseteq \cF$ is said to be an $\alpha$-cover for a metric space $(\cF,\kappa)$, if for every $f \in \cF$, there exists an $f' \in C_\alpha$ such that $\kappa(f,f')\leq \alpha$.
\end{definition}

\begin{definition} [Locally small cover \citep{bun2021private}] \label{def:localsc}
    Consider an $\alpha$-cover $C_{\alpha}$ for a  metric space $(\cF,\kappa)$. For $\gamma\geq \alpha$, $C_{\alpha}$ is said to be $(t, \gamma)$-locally small if:
\begin{equation*}
    sup_{f \in \cF} |B_\kappa(\gamma, f, C_\alpha)| \leq t
\end{equation*}
Moreover, if such a $\alpha$-cover exists, we say $\cF$ admits a $(t, \gamma)$-locally small $\alpha$-cover.
\end{definition}

\begin{definition}[$k$-mixtures] \label{def:k-mixtures}
    Let $\cF$ be an arbitrary class of distributions. We denote the class of $k$-mixtures of $\cF$ by $\kmix(\cF)= \Delta_k \times \cF^k$ where $\Delta_k = \{ w \in \bR^k \,:\, w_i \geq 0, \sum_{i=1}^k w_i = 1\}$ is the $(k-1)$-dimensional probability simplex.
\end{definition}
In this paper, we abuse notation and simply write $f \in \kmix(\cF)$ to denote a $k$-mixture from the class $\cF$ where the representation of the weights is implicit.
Further, if $f \in \kmix[k](\cF)$ and $g \in \kmix[k'](\cF)$ then we write $\dtv(f, g)$ to denote the TV distance from the underlying distribution.
In other words, if $f = \sum_{i \in [k]} w_i f_i$ and $g = \sum_{i \in [k']} w_i' g_i$ then $\dtv(f, g) = \dtv(\sum_{i \in [k]} w_i f_i, \sum_{i \in [k']} w_i' g_i)$.
The purpose of defining mixture distributions in this way is that it turns out to be convenient to define a distance between different representations of a distribution (see Definition~\ref{def:comp-dist}).

\begin{definition} [Unbounded Gaussians] \label{def:general_gaussian}
    Let $\cG=\{\cN(\mu, \Sigma): \mu \in \mathbb{R}^d, \Sigma \in S^d\}$ be the class of d-dimensional Gaussians.
\end{definition}

\begin{subsection} {Distribution learning and list decodable learning}
\end{subsection}
Here, we define list decodable learning under Huber's contamination model \citep{huber1992robust}, where the samples are drawn from a corrupted version of the original distribution $f$ that we are interested in. The contamination is additive, that is, with probability $1-\gamma$ we receive samples from $f$, and with probability $\gamma$ we receive samples from an arbitrary distribution $h$. Upon receiving corrupted samples from $f$, the goal of a list decoding algorithm is to output a short list of distributions one of which is close to the original distribution $f$. In the next definition, we use a general metric to measure the closeness; this allow for choosing the metric based on the application.

\begin{definition} [List decodable learning under Huber's contamination]
    Let $\cF$ be a class of distributions and $\kappa:\cF\times\cF\rightarrow \bR_{\geq 0}$ be a metric defined on it.
    For $\alpha,\beta,\gamma\in (0,1)$, an algorithm $A_{\cF}$ is said to be an $(L,m,\alpha,\beta,\gamma)$-list-decoding algorithm for $\cF$ w.r.t.~$\kappa$ if the following holds:

    \begin{quote}
    For any $f \in \cF$ and arbitrary distribution $h$,
    given an i.i.d.~sample $S$ of size $m$ from $g=(1-\gamma)f+\gamma h$,
    $A_\cF$ outputs a list of distributions $\hat{\cF} \subseteq \cF$ of size no more than $L$ such that with probability at least $1-\beta$ (over the randomness of $S$ and $A_\cF$) we have $\min_{\hat{f}\in \hat{\cF}} \kappa(\hat{f},f)\leq \alpha$.
    \end{quote}
    
\end{definition}

A distribution learner is a, possibly randomized, algorithm that receives i.i.d.~samples from a distribution $f$, and outputs a distribution $\hat{f}$ which is close to $f$. 
\begin{definition} [PAC learning]
    An algorithm $A_\cF$ is said to $(\alpha,\beta)$-PAC-learn a class of distributions $\cF$ w.r.t.~metric $\kappa$ with $m(\alpha, \beta)$ samples, if for any $f \in \cF$, and any $\alpha,\beta \in (0,1)$, after receiving $m(\alpha,\beta)$ i.i.d.~samples from $f$, outputs a distribution $\hat{f} \in \cF$ such that $\kappa(f,\hat{f})\leq \alpha$ with probability at least $1-\beta$. Moreover, if such an algorithm exists, we call $\cF$ to be $(\alpha,\beta)$-PAC-learnable w.r.t.~$\kappa$. The sample complexity of learning $\cF$ is the minimum $m(\alpha, \beta)$ among all such $(\alpha,\beta)$-PAC-learners.
\end{definition}

\begin{remark} \label{remark:pac-list}
 Equivalently, an algorithm $A_\cF$ is said to $(\alpha,\beta)$-PAC-learn a class of distributions $\cF$ w.r.t.~metric $\kappa$ with $m(\alpha, \beta)$ samples, if for any $\alpha,\beta\in(0,1)$, $A_\cF$ is a $(1,m(\alpha,\beta),\alpha,\beta,0)$-list-decoding algorithm for $\cF$ w.r.t.~$\kappa$.
\end{remark}

The following result on learning a finite class of distributions is based on the Minimum Distance Estimator~\citep{yatracos1985rates}; see the excellent book by~\citet{devroye2001combinatorial} for details.
\begin{theorem} [Learning finite classes, Theorem 6.3 of \citet{devroye2001combinatorial}] \label{thm:mde}
    Let $\alpha,\beta \in (0,1)$. Given a finite class of distributions $\cF$, there is an algorithm that upon receiving $O(\frac{\log|\cF|+\log(1/\beta)}{\alpha^2})$ i.i.d. samples from a distribution $g$, it returns an $\hat{f}\in \cF$ such that $\dtv(\hat{f},g)\leq 3\cdot\min_{f\in\cF}\dtv(f,g)+\alpha$ with probability at least $1-\beta$.
\end{theorem}

\begin{subsection}{Differential Privacy}
    Two datasets $D,D'\in \cX^n$ are called neighbouring datasets if they differ by one element. Informally, a differentially private algorithm is required to have close output distributions on neighbouring datasets.
    
    \begin{definition}[$(\eps,\delta)$-Indistinguishable]
        Two distribution $Y,Y'$ with support $\cY$ are said to be $(\eps,\delta)$-indistinguishable if for all measurable subsets $E \in \cY$, $\probs{X\sim Y}{X\in E} \leq e^\eps \probs{X\sim Y'}{X\in E} +\delta$ and $\probs{X\sim Y'}{X\in E} \leq e^\eps \probs{X\sim Y}{X\in E} +\delta$.
    \end{definition}
    \begin{definition}[$(\eps,\delta)$-Differential Privacy \citep{dwork2006calibrating,dwork2006our}]
        A randomized algorithm $\cM:\cX^n \to \cY$ is said to be $(\eps,\delta)$-differentially private if for every two neighbouring datasets $D,D'\in \cX^n$, the output distributions $\cM(D),\cM(D')$ are $(\eps,\delta)$-indistinguishable.
    \end{definition}
    
\end{subsection}
\end{section}

\section{Main Results}\label{sec:results}
In this section, we describe our main results. We introduce a general framework for privately learning mixture distributions, and as an application, we propose the first finite upper bound on the sample complexity of privately learning general GMMs. More specifically, we show that if we have \textit{(1)} a locally small cover (w.r.t.~$\dtv$), and \textit{(2)} a list decoding algorithm (w.r.t.~$\dtv$) for a class of distributions, then the class of its mixtures is privately learnable.

\begin{restatable}[Reduction]{theorem}{mainreduction}
\label{thm:main}
        For $\alpha,\beta \in (0,1)$, if a class of distributions $\cF$ admits a $(t,2\alpha/15)$-locally small $\frac{\alpha}{15}$-cover (w.r.t.~$\dtv$), and it is $\left(L, m,\alpha/15,\beta',0\right)$-list-decodable (w.r.t.~$\dtv$), where $\log(1/\beta') = \widetilde{\Theta}(\log(mk\log(tL/\alpha\delta) / \eps\beta))$, then $\kmix(\cF)$ is $(\eps,\delta)$-DP $(\alpha,\beta)$-PAC-learnable (w.r.t.~$\dtv$)
        with sample complexity
        \[
        \tilde{O}\left(\left(\frac{\log(1/\delta)+k\log(tL)}{\eps}+\frac{mk+k\log(1/\beta)}{\alpha\eps}\right)\cdot \left(\frac{k\log(L)}{\alpha^2}+\frac{mk+k\log(1/\beta)}{\alpha^3}\right)\right).
        % \tilde{O}\left(\frac{\log(1/\delta)+mk+k\log(tL)+\log(1/\beta)}{\eps}\cdot \frac{mk+k\log(L)+k\log(1/\beta)}{\alpha^2}\right).
        \]
\end{restatable}

Note that in Theorem~\ref{thm:main}, we can use a naive $(\alpha,\beta)$-PAC-learner that outputs \textit{a single} distribution as the list decoding algorithm (see Remark~\ref{remark:pac-list}). Therefore, if we have \textit{(1)} a locally small cover (w.r.t.~$\dtv$), and \textit{(2)} a (non-private) PAC learner (w.r.t.~$\dtv$) for a class of distributions, then the class of its mixtures is privately learnable\footnote{Later in Remark~\ref{remark:gmm-pac-instead}, we explain how Theorem~\ref{thm:main} can sometimes give us a better bound compared to Corollary~\ref{cor:pac-instead}.}. The next corollary states this result formally. 
\begin{corollary} \label{cor:pac-instead}
     For $\alpha,\beta \in (0,1)$, if a class of distributions $\cF$ admits a $(t,2\alpha/15)$-locally small $\frac{\alpha}{15}$-cover (w.r.t.~$\dtv$), and it is $(\alpha/15,\beta')$-PAC-learnable (w.r.t.~$\dtv$) using $m(\alpha/15,\beta')$ samples, where $\log(1/\beta') = \widetilde{\Theta}(\log(mk \log(t/\alpha \delta)/\eps \beta))$, then $\kmix(\cF)$ is $(\eps,\delta)$-DP $(\alpha,\beta)$-PAC-learnable (w.r.t.~$\dtv$)
     with sample complexity
    \[
\tilde{O}\left(\left(\frac{\log(1/\delta)}{\eps}+\frac{m(\alpha/15,\beta')k+k\log(1/\beta)}{\alpha\eps}\right)\cdot \left(\frac{m(\alpha/15,\beta')k+k\log(1/\beta)}{\alpha^3}\right)\right).    
    % \tilde{O}\left(\frac{\log(1/\delta)+m(\alpha/15,\beta')k+k\log(t)+\log(1/\beta)}{\eps}\cdot \frac{m(\alpha/15,\beta')k+k\log(1/\beta)}{\alpha^2}\right)
    \]
\end{corollary}

As an application of the Theorem~\ref{thm:main}, we show that the class of GMMs is privately learnable. We need two ingredients to do so. We show that the class of unbounded Gaussians \textit{(1)} has a locally small cover, and \textit{(2)} is list decodable (using compression).

As a result, we prove the first sample complexity upper bound for privately learning general GMMs. Notably, the above upper bound is polynomial in all the parameters of interest. 

\begin{restatable}[Private Learning of GMMs]{theorem}{maingmm}\label{thm:main-gmm}
     Let $\alpha,\beta \in (0,1)$. The class $\kmix(\cG)$ is $(\eps,\delta)$-DP $(\alpha,\beta)$-PAC-learnable w.r.t.~$\dtv$
    with sample complexity
    \[
    \tilde{O}\left(\frac{kd^2\log(1/\delta)+k^2d^4}{\alpha^2\eps}+ \frac{kd\log(1/\delta)\log(1/\beta)+k^2d^3\log(1/\beta)}{\alpha^3\eps}+\frac{k^2d^2\log^2(1/\beta)}{\alpha^4\eps}\right).
    % \tilde{O}\left(\frac{k^2d^4+k^2d^3\log(1/\beta)+kd^2\log(1/\delta) + k^2d^2\log^2(1/\beta)+kd\log(1/\beta)\log(1/\delta)}{\alpha^2\eps}\right).
    \]
\end{restatable}

\begin{remark} \label{remark:gmm-pac-instead}
Note that if we had used Corollary~\ref{cor:pac-instead} and a PAC learner as a naive list decoding algorithm for Gaussians, 
the resulting sample complexity would have become worse. To see this, note that $\cG$ is $(\alpha,\beta)$-PAC-learnable using $m(\alpha,\beta)=O(\frac{d^2\log(1/\beta)}{\alpha^2})$ samples. Using Corollary~\ref{cor:pac-instead} and the existence of a locally small cover for Guassians, we obtain a sample complexity upper bound of
    \[
    \tilde{O}\left(\frac{kd^2\log(1/\delta)\log(1/\beta)}{\alpha^5\eps}+\frac{k^2d^4\log^2(1/\beta)}{\alpha^8\eps}\right).
 % \tilde{O}\left( \frac{k^2d^4\log^2(1/\beta)}{\alpha^6\eps}+\frac{kd^2\log(1/\beta)\log(1/\delta)}{\alpha^4\eps}\right).
    \]
    This is a weaker result compared to Theorem~\ref{thm:main-gmm} in terms of $\alpha$, which was based on a more sophisticated (compression-based) list decoding algorithm for Gaussians. 
\end{remark}

It is worth mentioning that our approach does not yield a finite time algorithm for privately learning GMMs, due to the non-constructive cover that we use for Gaussians. Moreover, designing a computationally efficient algorithm (i.e.~with a running time that is polynomial in $k$ and $d$) for learning GMMs even in the non-private setting remains an open problem \citep{diakonikolas2017statistical}.

\section{Technical Challenges and Contributions}\label{sec:techniques}

{\bf Dense mixtures.} As a simple first step, we reduce the problem of learning mixture distributions to the problem of learning ``dense mixtures'' (i.e., those mixtures whose component weights are not too small). Working with dense mixtures is more convenient since a large enough sample from a dense mixture will include samples from \emph{every} component. 

{\bf Locally small cover for GMMs w.r.t.~$\dtv$?} One idea to privately learn (dense) GMMs is to create a locally small cover w.r.t.~$\dtv$ (see Definition~\ref{def:localsc}) for this class and then apply ``advanced'' private hypothesis selection~\citep{bun2021private}. However, as \citet{aden2021privately} showed, such a locally small cover (w.r.t.$\dtv$) does \emph{not} exists, even for a mixtures of two (dense) Gaussians.

{\bf A locally small cover for the component-wise distance.} An alternative measure for the distance between two mixtures is their component-wise distance, which we denote by $\kappamix$ (see Definition~\ref{def:comp-dist}).
Intuitively, given two mixtures, $\kappamix$ measures the distance between their farthest components. Therefore, if two GMMs are close in $\kappamix$ then they are close in $\dtv$ distance too. Interestingly, we prove that GMMs \textit{do} admit a locally small cover w.r.t.~$\kappamix$. To prove this, we first show that if a class of distributions admits a locally small cover w.r.t.~$\dtv$ then the class of its mixtures admits a locally small cover w.r.t.~$\kappamix$. Next, we argue that the class of Gaussians admits a locally small cover w.r.t.~$\dtv$. Building a locally small cover for the class of Gaussians is challenging due to the complex geometry of this class. We show the existence of such cover using the techniques of \citet{aden2021sample} and the recently proved lower bound for the $\dtv$ distance between two (high dimensional) Gaussians~\citep{arbas2023polynomial}.

{\bf Hardness of learning GMMs w.r.t.~$\kappamix$.} Given that we have a locally small cover for GMMs w.r.t.~$\kappamix$, one may hope to apply some ideas similar to private hypothesis selection for privately learning GMMs w.r.t.~$\kappamix$. Unfortunately, learning GMMs w.r.t.~$\kappamix$, even in the non-private setting, requires exponentially many samples in terms of the number of components~\citep{moitra2010settling}.

{\bf List decoding (dense mixtures) w.r.t.~$\kappamix$.} Interestingly, we show that unlike \emph{PAC learning}, \emph{list decoding} GMMs w.r.t.~$\kappamix$ can be done with a polynomial number of samples. To show this, first, we prove that if a class of distributions is list decodable (w.r.t.~$\dtv$), then class of its dense mixtures is list decodable (w.r.t.~$\kappamix$). Then for the class of Guassians, we use a compression-based~\citep{ashtiani2018nearly} list decoding method.

{\bf Privacy challenges of using the list decoder.} Unfortunately, the list decoding method we described is not private. Otherwise, we could have used Private Hypothesis Selection~\citep{bun2021private} to privately choose from the list of candidate GMMs. To alleviate this problem, we define and solve the ``private common member selection'' problem below.

{\bf Private common member selection.} Given $T$ lists of objects (e.g., distributions), we say an object is a common member if it is close (w.r.t.~some metric $\kappa$) to a member in each of the lists (we give a rigorous definition Section~\ref{sec:pcms}). The goal of a private common member selector (PCMS) is then to privately find a common member assuming at least one exists. We then show \textit{(1)} how to use a PCMS to learn GMMs privately and \textit{(2)} how to solve the PCMS itself. This will conclude the proof of Theorem~\ref{thm:main}.

{\bf Private learning of GMMs using PCMS.} Given a PCMS, we first run the (non-private) list decoding algorithm on $T$ disjoint datasets to generate $T$ lists of dense mixture distributions. At this point, we are guaranteed that with high probability, there exists a common member for these lists w.r.t.~$\kappamix$. Therefore, we can simply run the PCMS method to find such a common member. However, note that not all the common members are necessarily ``good'': there might be some other common members that are far from the true distribution w.r.t.~$\dtv$. To resolve this issue, in each list we filter out (non-privately) the distributions that are far from the true distribution. After filtering the lists, we are still guaranteed to have a ``good'' common member and therefore we can run PCMS to choose it privately.

{\bf Designing a PCMS for locally small spaces.} Finally, we give a recipe for designing a private common member selector for $T$ lists w.r.t. a generic metric~$\kappa$.
To do so, assume we have access to a locally small cover for the space w.r.t.~$\kappa$ (indeed, we had showed this is the case for the space of GMMs w.r.t.~$\kappamix$). We need to privately choose a member from this cover that represents a common member. We then design a score function such that: \textit{(1)} a common member gets a high score and \textit{(2)} the sensitivity of the score function is low (i.e., changing one of the input \emph{lists} does not change the score of any member drastically). Using this score function, we apply the GAP-MAX algorithm of \citet{bun2021private,bun2018composable} to privately select a member with a high score from the infinite (but locally small) cover.

\begin{section}{Common Member Selection}\label{sec:pcms}
In this section, we introduce the problem of Common Member Selection, which is used in our main reduction for privately learning mixture distributions.
At a high-level, given a set of lists, a common member is an element that is close to some element in most of the lists. The problem of common member selection is to find such an item.

\begin{definition}[Common Member]
     Let $(\cF,\kappa)$ be a metric space and $\alpha,\zeta \in (0,1]$. We say $f\in \cF$ is an $(\alpha,\zeta)$-common-member for $\cY =\{Y_1,Y_2,\ldots,Y_T\} \in (\cF^*)^T$, if there exists a subset $\cY' \subseteq \cY$ of size at least $\zeta T$, such that $\max_{Y\in \cY'}\kappa(f,Y)\leq \alpha$.
\end{definition}

\begin{definition}[Common Member Selector (CMS)] \label{def:cms}
     Let $(\cF,\kappa)$ be a metric space, and $\alpha,\zeta,\beta \in (0,1]$. An algorithm $\cA$ is said to be a $(T_0,Q,\alpha,\zeta,\beta)$-common-member selector w.r.t.~$\kappa$ if the following holds for all $T \geq T_0$:
    \begin{quote}
    Given any $\cY =\{Y_1,Y_2,...,Y_T\} \in (\cF^*)^T$ that satisfies $|Y_i|\leq Q$ for all $i\in [T]$, if there exists at least one $(\alpha,1)$-common-member for $\cY$, then $\cA$ outputs 
    a $(2 \alpha,\zeta)$-common-member with probability at least $1-\beta$.
    \end{quote}

\end{definition}

\begin{remark}
    Note that the CMS problem on its own is a trivial task and can be done using a simple brute force algorithm. However, we are interested in the non-trivial privatized version of this problem. The formal definition of the private CMS is give below.
\end{remark}

\begin{definition}[Private Common Member Selector (PCMS)]
      Let $(\cF,\kappa)$ be a metric space. Further, let $\alpha,\beta, \zeta, \delta \in (0,1]$ and $\eps \geq 0$ be parameters. An algorithm $\cA$ is an $(\eps,\delta)$-DP $(T_0,Q,\alpha,\zeta,\beta)$-common-member selector w.r.t.~$\kappa$ if (1) it is a $(T_0,Q,\alpha,\zeta,\beta)$-CMS and (2) for any $T \geq T_0$ and any two collections of lists $C_1=\{Y_1,Y_2,...,Y_T\}\in (\cF^*)^T$ and $C_2=\{Y'_1,Y_2,...,Y_T\}\in (\cF^*)^T$ that differ in only one list, the output distributions of $\cA(C_1)$ and $\cA(C_2)$ are $(\eps,\delta)$-indistinguishable. 
\end{definition}

In Algorithm~\ref{alg:pcms}, we describe an algorithm for privately fining a common member provided that one exists. We note that one requirement is that we have access to a locally small cover $\cC$ for $\cF$. At a high-level, given a family of sets $\{Y_1, \ldots, Y_T\}$, where each $Y_t$ is a set of elements, we can assign a score to each point $c \in \cC$ to be the number of $Y_t$'s that contain a element close to $c$. We observe that the sensitivity of this score function is 1, meaning that by changing one set, the score of each point will change by at most 1. We remark that this is because the score of a point c is defined to be the number of lists that contains an element close to it, not the total number of elements. Note that any point $c$ with a sufficiently high score is a common member. Further, since $\cC$ is locally small, this allows us to apply the GAP-MAX algorithm (Theorem~\ref{thm:gap-max}) to make this selection differentially private \footnote{This task can also be done using the Choosing Mechanism of \citet{bun2015differentially}.}.
In Theorem~\ref{thm:pcms}, we prove the correctness of this algorithm.

\begin{algorithm}
\caption{Private Common Member Selector (PCMS)}\label{alg:pcms}
\begin{algorithmic}[1]
\Require $D=\{Y_1,Y_2,...,Y_T\} \in (\cF^*)^T, $ metric $\kappa$ over $\cF$, $(t,2\alpha)$-locally-small $\alpha$-cover $C_{\alpha}$ for $\cF$ w.r.t.~$\kappa$.
\Ensure $(2\alpha,0.9)$-common-member of $D$ (assuming $D$ has an $(\alpha,1)$-common-member)
\State For all $h \in C_{\alpha}$, set $\score(h, D) \coloneqq |\{i \in [T]\,:\, \kappa(y, h) \leq 2\alpha \text{ for some $y \in Y_i$}\}|$.
\State \Return GAP-MAX($C_{\alpha}, D, \score, 0.1, \beta$)
\end{algorithmic}
\end{algorithm}

\begin{theorem}
\label{thm:pcms}
Let $(\cF,\kappa)$ be a metric space, $\alpha,\beta,\delta \in (0,1]$, $Q\in \bN$, $\eps > 0$, and $C_{\alpha}$ be a $(t,2\alpha)$-locally-small $\alpha$-cover  for $\cF$ (w.r.t.~$\kappa$).
Algorithm~\ref{alg:pcms} is an $(\eps,\delta)$-DP $(T,Q,\alpha,0.9,\beta)$-common-member selector w.r.t.~$\kappa$ for some $T = O\left(\frac{\log(1/\delta)+\log(tQ/\beta)}{\eps}\right)$.
\end{theorem}

A known tool for privately choosing a ``good'' item from a set of candidates is Exponential Mechanism \citep{mcsherry2007mechanism}, where the ``goodness'' of candidates is measured using a score function. However, Exponential Mechanism fails in the regimes where the candidate set is not finite. This leads us to use the GAP-MAX algorithm of \citep{bun2021private,bun2018composable} that has the advantage of compatibility with infinite candidate sets. GAP-MAX guarantees returning a ``good'' candidate as long as the number of ``near good'' candidates is small.

\begin{theorem}[GAP-MAX, Theorem IV.6. of \cite{bun2021private}] \label{thm:gap-max}
Let $(\cF,\kappa)$ be a metric space and $\cX$ be an arbitrary set. Let $\score\colon \cF \times \cX^T \to \bR_{\geq 0}$ be a function such that for any $f\in \cF$ and any two neighbouring sets $D\sim
D'\in \cX^T$, we have $|\score(f,D)-\score(f,D')|\leq 1$. Then there is an algorithm, called the GAP-MAX algorithm, that is $(\eps,\delta)$-DP with the following property. For every $D \in \cX^T$ and $\alpha' \in (0,1)$, if
\begin{equation*}
    \left|\left\{f \in \cF: \score(f,D) \geq \sup_{f' \in \cF} \score(f',D) - 5\alpha' T\right\}\right| \leq t
\end{equation*}
then
\begin{align*}
\prob{\score(\text{GAP-MAX}(\cF,D,\score,\alpha',\beta),D) \geq \sup_{f' \in \cF}\score(f',D) - \alpha' T} \geq 1-\beta
\end{align*}
provided $T = \Omega\left(\frac{\min\{\log|\cF|,\log(1/\delta)\}+\log(t/\beta)}{\alpha'\eps}\right)$.
\end{theorem}

\begin{proof}[Proof of Theorem~\ref{thm:pcms}]
We first prove the utility of the algorithm.

\textbf{Utility.}
We show that the output of the Algorithm~\ref{alg:pcms} is an $(2\alpha,0.9)$-common-member provided that there exists an $(\alpha,1)$-common-member (recall Definition~\ref{def:cms}).
Let the score function be defined as in Algorithm~\ref{alg:pcms}. Since $C_{\alpha}$ is $(t,2\alpha)$-locally-small, we have
\begin{equation*}
\left|\left\{h \in C_{\alpha}: \score(h,D) \geq 1\right\}\right| \leq tQT   
\end{equation*}
since for any $i\in[T]$ and any $y\in Y_i$, $y$ contributes to at most $t$ candidates' scores. As a result, there are at most $tQT$ candidates with non-zero scores. Assuming that there exists an $(\alpha,1)$-common-member for $\{Y_1,Y_2,...,Y_T\}$, we have $\sup_{h \in C_{\alpha}} \score(h,D)=T$. Thus,
\begin{align*}
&\left|\left\{h \in C_{\alpha}: \score(h,D) \geq \sup_{h \in C_{\alpha}} \score(h,D) - T/2\right\}\right| \\
=~&\left|\{h \in C_{\alpha}: \score(h,D) \geq T - T/2\}\right| \\
=~&\left|\{h \in C_{\alpha}: \score(h,D) \geq T/2\}\right| \\
\leq~&\left|\{h \in C_{\alpha}: \score(h,D) \geq 1\}\right| \leq tQT.
\end{align*}
Using this bound, we can apply GAP-MAX algorithm in Theorem~\ref{thm:gap-max} with $\alpha' = 0.1$, and $T=O\left(\frac{\log(1/\delta)+\log(tQ/\beta)}{\eps}\right)$. In particular, if $\hat{h}=\text{GAP-MAX}(C_\alpha,D,\score,0.1,\beta)$ then
\begin{align*}
    \prob{\score(\hat{h},D) \geq 0.9T}
    =\prob{\left|\left\{i\in[T]\,:\,\exists y \in Y_i \text{ such that } \kappa(y,\hat{h})\leq 2\alpha \right\}\right| \geq 0.9T}
    \geq 1-\beta.
\end{align*}

\textbf{Privacy.}
    Note that for any $h\in C_{\alpha}$ and any two neighbouring sets $D\sim D'$, we have $|\score(h,D)-\score(h,D')|\leq 1$ since each list $y\in D$ contributes to any $h$'s score by at most 1. Thus, Theorem~\ref{thm:gap-max} implies that GAP-MAX is $(\eps,\delta)$-DP.
\end{proof}

\end{section}

\section{Mixtures and Their Properties}
In this section, we study some general properties of mixture distributions. First, we introduce a component-wise distance between two mixture distributions which will be useful for constructing locally small covers. Generally, if we have a locally small cover for a class of distributions w.r.t.~$\dtv$, then there exists a locally small cover w.r.t.~component-wise distance for mixtures of that class. Later, we define dense mixtures and will show that if a class of distributions is list decodable w.r.t.~$\dtv$, then the dense mixtures of that class are list decodable w.r.t.~component-wise distance.

\subsection{Component-wise distance between mixtures}
Here, we define the class of general mixtures which are the mixtures with arbitrary number of components, as opposed to Definition~\ref{def:k-mixtures}, where the number of components is fixed.
\begin{definition} [general mixtures]
    Let $\cF$ be an arbitrary class of distributions. We denote the class of mixtures of $\cF$ by $\mix(\cF)=\bigcup_{k=1}^{\infty}$ $\kmix(\cF)$.
\end{definition}

Below, we define the component-wise distance between two mixture distributions with arbitrary number of components. The definition is inspired by \citet{moitra2010settling}. We set the distance between two mixtures with different number of components to be $\infty$. Otherwise, the distance between two mixtures is the distance between their farthest components.

\begin{definition} [Component-wise distance between two mixtures] \label{def:comp-dist}
    For a class $\cF$ and every $g_1=\sum_{i\in[k_1]}w_if_i \in \kmix[k_1](\cF)$, $g_2=\sum_{i\in[k_2]}w'_if'_i \in \kmix[k_2](\cF)$, we define the distance $\kappa_{mix}\colon \mix(\cF) \times \mix(\cF) \to \bR_{\geq 0}$ as
    \begin{equation}
        \kappa_{mix}(g_1,g_2) = \begin{cases}
            \min_\pi \max_{i\in [k_1]} \max \{k_1.|w_i-w'_{\pi(i)}|, \dtv(f_i,f'_{\pi(i)})\} & k_1=k_2\\
            \infty & k_1\neq k_2
        \end{cases}
    \end{equation}
    where $\pi$ is chosen from all permutations over $[k_1]$.
\end{definition}

The next lemma states that if two mixture distributions are close w.r.t.~$\kappa_{mix}$, then they are also close w.r.t.~$\dtv$.
\begin{lemma} \label{lemma:kmix-to-tv}
    Let $\alpha \in [0,1]$ and $f = \sum_{i\in[k]}w_if_i,f'=\sum_{i\in[k]}w_i'f_i' \in \kmix(\cF)$. If $\kappa_{mix}(f,f')\leq \alpha$, then $\dtv(f,f')\leq 3\alpha/2$.
\end{lemma}
\begin{proof}
    Using the definition of $\kappa_{mix}$, we get that for every $i\in[k]$, $|w_i-w_i'|\leq \alpha/k$ and $\dtv(f_i,f_i')\leq \alpha$. Therefore,
    \begin{align*}
        \dtv(f,f')= \frac{1}{2} ||f-f'||_1 &= \frac{1}{2} ||\sum_{i\in[k]}w_if_i - \sum_{i\in[k]}w_i'f_i'||_1 \\
        &\leq
        \frac{1}{2} \sum_{i\in[k]} ||w_if_i-w_i'f_i'||_1 \\
        &\leq
        \frac{1}{2} \sum_{i\in[k]} ||w_if_i-w_i'f_i||_1+ ||w_i'f_i-w_i'f_i'||_1 \\
        &\leq 
        \frac{1}{2} \sum_{i\in[k]} ||w_if_i-w_i'f_i||_1+ ||w_i'f_i-w_i'f_i'||_1 \\
        &\leq
        \frac{1}{2} \sum_{i\in[k]} \frac{\alpha}{k}+ \frac{1}{2} \sum_{i\in[k]} 2\alpha w_i' =3\alpha/2. \qedhere
    \end{align*}
\end{proof}

The following simple proposition gives a locally small cover for weight vectors used to construct a mixture. 
\begin{proposition}\label{lemma:weightcover}
    Let $\alpha \in (0,1]$. There is an $\alpha$-cover for $\Delta_k = \{(w_1,w_2,...,w_k)\in \bR_{\geq 0}^k\,:\,\sum_{i\in [k]}w_i=1\}$ w.r.t.~$\ell_\infty$ of size at most $(1/\alpha)^k$.
\end{proposition}
\begin{proof}
    Partition the cube $[0,1]^k$ into small cubes of side-length $1/\alpha$. If for a cube $c$, we have $c \cap \Delta_k \neq \emptyset$, put one arbitrary point from $c \cap \Delta_k$ into the cover. The size of the constructed cover is no more than $(1/\alpha)^k$ which is the total number of small cubes.
\end{proof}

The next lemma states that if a class of distributions has a locally small cover w.r.t.~$\dtv$, then the mixtures of that class admit a locally small cover w.r.t.~$\kappa_{mix}$. Note that the choice of the metric is important as the next theorem is false if we consider the $\dtv$ metric for mixtures. In other words, there is a class of distributions (e.g.~Gaussians) that admits a locally small cover w.r.t.~$\dtv$ but there is no locally small cover for the mixtures of that class w.r.t.~$\dtv$ (Proposition 1.3 of \citet{aden2021privately}).

\begin{theorem} \label{thm:kimix-cover}
For any $0< \alpha < \gamma < 1$, if a class of distributions $\cF$ has a $(t,\gamma)$-locally-small $\alpha$-cover w.r.t.~$\dtv$, then the class $\kmix(\cF)$ has a $(k!(tk/\alpha)^k,\gamma)$-locally-small $\alpha$-cover w.r.t.~$\kappa_{mix}$.
\end{theorem}
\begin{proof}
Let $C_\alpha$ be the $(t,\gamma)$-locally small $\alpha$-cover for $\cF$, and $\hat{\Delta}_k$ be an $\frac{\alpha}{k}$-cover for the probability simplex $\Delta_k$ from Proposition~\ref{lemma:weightcover}. Construct the set $\cJ=\{\sum_{i\in[k]}\hat{w_i}\hat{f_i}: \hat{w}\in \hat{\Delta}_k, \hat{f_i}\in C_\alpha\}$. Note that $\cJ$ is an $\alpha$-cover for $\kmix(\cF)$ w.r.t.~$\kappa_{mix}$ since for any $g=\sum_{i\in[k]}w_if_i \in$ $\kmix(\cF)$, by construction, there exists an $g'\in \cJ$ such that $\kappa_{mix}(g,g')\leq \alpha$. Moreover, we have $|B_{\kappa_{mix}}(\gamma,g,\cJ)|\leq|B_{\dtv}(\gamma,f_i,C_\alpha)|^k \cdot|\hat{\Delta}_k|\cdot k!= t^k \cdot (k/\alpha)^k \cdot k!$, where the first term is because of composing the cover for a single component $k$ times. The term $|\hat{\Delta}_k|$ comes from the size of cover for mixing weights of $k$ components, and the $k!$ term is the result of permutation over $k$ unordered components in the mixture. 
\end{proof}

\subsection{Dense mixtures}
Dense mixtures are mixture distributions where each component has a non-negligible weight. 
Intuitively, a dense mixture is technically easier to deal with since given a large enough sample from the dense mixture, one would get samples from \emph{all} of the components.
This will allow us to show that if a class of distribution is list decodable w.r.t.~$\dtv$, then the class of its dense mixtures is list decodable w.r.t.~$\kappa_{mix}$.
Later, we reduce the problem of learning mixture distributions to the problem of learning dense mixtures.

\begin{definition}[Dense mixtures]
    Let $\cF$ be an arbitrary class of distributions, $k\in \bN$, and $\eta \in [0,1/k]$. We denote the class of k-mixtures of $\cF$ without negligible components by $\densemix{(k, \eta)}(\cF)=\{\sum_{i=1}^{s} w_i f_i: s\leq k, w_i\geq \eta, \sum_{i=1}^{s} w_i =1, f_i \in \cF\}$.
\end{definition}

The next lemma states that every mixture distribution can be approximated using a dense mixture. 
\begin{lemma} \label{lemma:dense}
     For every $k\in \bN$, $g\in \kmix(\cF)$ and $\alpha \in [0,1)$, there exists $\gamma \in [0,\alpha)$, $g' \in \densemix{(k, \alpha/k)}(\cF)$, and a distribution $h$ such that $g=\gamma h + (1-\gamma)g' $.
\end{lemma}
\begin{proof}
     For any $g=\sum_{i\in[k]}w_if_i\in \kmix(\cF)$, let $N=\{i\in[k]: w_i < \alpha/k\}$ be the set of negligible weights, and $\gamma = \sum_{i\in N} w_i < \alpha$. Then $g$ can be written as $g=(1-\gamma) \sum_{i\in [k]\setminus N} \dfrac{w_i}{1-\gamma}f_i + \gamma \sum_{i\in  N} \dfrac{w_i}{\gamma}f_i$.
     Note that $\sum_{i\in [k]\setminus N} \dfrac{w_i}{1-\gamma}f_i \in \densemix{(k,\frac{\alpha}{k})}(\cF)$.
\end{proof}

Theorem~\ref{thm:kimix-cover} shows that if a class of distributions admits a locally small cover (w.r.t.~$\dtv$) then the class of its mixtures admits a locally small cover (w.r.t.~$\kappa_{mix}$). In the next lemma, we see that this is also the case for dense mixtures, i.e.~if a class of distributions admits a locally small cover (w.r.t.~$\dtv$) then the class of its dense mixtures admits a locally small cover (w.r.t.~$\kappa_{mix}$).
\begin{lemma}\label{lemma:densemix-cover}
    For any $0< \alpha < \gamma < 1$, and $\alpha'\in(0,1]$, if a class of distributions $\cF$ has a $(t,\gamma)$-locally-small $\alpha$-cover w.r.t.~$\dtv$, then the class $\densemix{(k,\frac{\alpha'}{k})}(\cF)$ has a $(k!(tk/\alpha)^k,\gamma)$-locally-small $\alpha$-cover w.r.t.~$\kappa_{mix}$.
\end{lemma}
\begin{proof}
Using Theorem~\ref{thm:kimix-cover} we know that if $\cF$ has a $(t,\gamma)$-locally-small $\alpha$-cover w.r.t.~$\dtv$, then for every $i \in [k]$, there exists an $(i! (ti/\alpha)^i,\gamma)$-locally-small $\alpha$-cover $C_i$ for $\kmix[i](\cF)$ w.r.t.~$\kappa_{mix}$. Since $\densemix{(k,\frac{\alpha'}{k})}$($\cF$) $\subseteq \bigcup_{i \in [k]} \kmix[i](\cF)$, we get that $\cJ = \bigcup_{i\in[k]}C_i$ is an $\alpha$-cover for $\densemix{(k,\frac{\alpha'}{k})}(\cF)$. Moreover, $\cJ$ is $(k! (tk/\alpha)^k,\gamma)$-locally-small since the $\kappa_{mix}$ distance is $\infty$ for two mixtures with different number of components. 
\end{proof}

The following theorem is one of the main ingredients used for reducing the problem of privately learning mixtures to common member selection. It states that if a class of distributions is list decodable (w.r.t.~$\dtv$), then the class of its dense mixtures is list decodable (w.r.t.~$\kappa_{mix}$).

\begin{theorem} \label{thm:list-mix}
For any $\alpha,\beta,\gamma\in(0,1)$, if a class of distributions $\cF$ is $(L,m,\alpha,\beta,1-\alpha/k)$-list-decodable w.r.t.~$\dtv$, then the class $\densemix{(k,\frac{\alpha}{k})}(\cF)$ is $\left(L', m',\alpha,2k\beta,\gamma\right)$-list-decodable w.r.t.~$\kappa_{mix}$, where $L'=(\frac{kL}{\alpha})^{k+1}\cdot (\frac{10e\log(1/k\beta)}{1-\gamma})^m$, and $m'=\frac{2m+8\log(1/k\beta)}{1-\gamma}$.
\end{theorem}

In order to prove Theorem~\ref{thm:list-mix}, we need the following lemma, which states that if a class of distributions is list decodable with contamination level $\gamma=0$, it is also list decodable with $\gamma > 0$, at the cost of additional number of samples \emph{and an increased list size}.
\begin{lemma}\label{lemma:list-decode-inc}
    For any $\alpha,\beta,\gamma \in (0,1)$, if a class of distributions $\cF$ is $(L,m,\alpha,\beta,0)$-list-decodable w.r.t.~$\kappa$, then it is $\left(L(\frac{10e\log(1/\beta)}{1-\gamma})^m,\frac{2m+8\log(1/\beta)}{1-\gamma},\alpha,2\beta,\gamma\right)$-list-decodable w.r.t.~$\kappa$. 
\end{lemma}
\begin{proof}
    Let $f \in \cF$ and $h$ be an arbitrary distribution. Consider $g=(1-\gamma)f+\gamma h$, where $\gamma \in (0,1)$. Upon drawing $N$ samples from $g$, let $X_N$ be the random variable indicating the number of samples coming from $f$. Note that $X_N$ has binomial distribution. Setting $N\geq\frac{2m+8\log(1/\beta)}{1-\gamma}$, results in $\expect{X_N}/2 \geq m$, $\expect{X_N} \geq 8 \log(1/\beta)$. Using the Chernoff bound (Theorem 4.5(2) of \citep{mitzenmacher2005probability}), we have $\prob{X_N \leq m} \leq \prob{X_N \leq \expect{X_N}/2} \leq \exp(-\expect{X_N} / 8) \leq \beta$.
    Meaning that after drawing $N\geq\frac{2m+8\log(1/\beta)}{1-\gamma}$ samples from $g$, with probability at least $1-\beta$, we will have $m$ samples coming from $f$, which is enough for list decoding $\cF$. Let $S_1,\cdots,S_{K}$ be all subsets of these $N$ samples with size $m$, where $K = {N\choose m}$. Now, run the list decoding algorithm on these subsets and let $\cL_i$ be the outputted list. Let $\cL = \cup_{i\in[K]} \cL_i$. Using the fact that among $N$ samples there are $m$ samples from $f$, we get that there exists $\hat{f} \in \cL$ such that with probability at least $1-\beta$, we have $\kappa(f,\hat{f})\leq \alpha$. Note that using Stirling's approximation we have $|\cL|=L\cdot{N\choose m} \leq L\cdot(\frac{2em+8e\log(1/\beta)}{(1-\gamma)m})^m \leq L\cdot(\frac{10e\log(1/\beta)}{1-\gamma})^m$.
    Finally, using a union bound we will get that $\cF$ is $\left(L(\frac{10e\log(1/\beta)}{1-\gamma})^m,\frac{2m+8\log(1/\beta)}{1-\gamma},\alpha,2\beta,\gamma\right)$-list-decodable w.r.t.~$\kappa$.
\end{proof}

\begin{proof}[Proof of Theorem~\ref{thm:list-mix}]
 Consider the algorithm $\cA$ to be an $(L,m,\alpha,\beta,1-\alpha/k)$-list-decodable learner for $\cF$.
Fix any distribution $g=\sum_{i\in[s]}w_if_i \in \densemix{(k,\frac{\alpha}{k})}(\cF)$, where $s\leq k$. Note that for any $i \in [s]$, $g$  can be written as $g = w_if_i + (1-w_i)\sum_{j\neq i}\frac{w_jf_j}{1-w_j}=w_if_i + (1-w_i)h$. Knowing that $w_i\geq \frac{\alpha}{k}$, allows us to apply the algorithm $\cA$ on the $m$ samples generated from $g$ and get a list of distributions $\cL_s$ such that with probability at least $1-\beta$ we have $\min_{f'\in \cL_s} \dtv(f',f_i)\leq \alpha$. Let $\hat{\Delta}_s$ be an $\frac{\alpha}{s}$-cover for $\Delta_s$ from Proposition~\ref{lemma:weightcover}. Now construct a set $\cJ=\bigcup_{s\in[k]}\{\sum_{i\in[s]}\hat{w_i}\hat{f_i}: \hat{w}\in \hat{\Delta}_s, \hat{f_i}\in \cL_s\}$. Note that with probability at least $1-k\beta$ we have $\min_{g'\in \cJ}\kappa_{mix}(g,g')\leq \max\{\alpha, \alpha\} = \alpha$. Moreover, we have $|\cJ|=\sum_{s\in[k]}|\cL_s|^s |\hat{\Delta}_s| = \sum_{s\in[k]} L^s(\frac{s}{\alpha})^s \leq (\frac{kL}{\alpha})^{k+1}$. Thus, the class of $\densemix{(k,\frac{\alpha}{k})}(\cF)$ is $\left((\frac{kL}{\alpha})^{k+1}, m,\alpha,k\beta,0\right)$-list-decodable w.r.t.~$\kappa_{mix}$. Using Lemma~\ref{lemma:list-decode-inc}, we get that $\densemix{(k,\frac{\alpha}{k})}(\cF)$ is $\left((\frac{kL}{\alpha})^{k+1}\cdot (\frac{10e\log(1/k\beta)}{1-\gamma})^m, \frac{2m+8\log(1/k\beta)}{1-\gamma},\alpha,2k\beta,\gamma\right)$-list-decodable w.r.t.~$\kappa_{mix}$.
\end{proof}
\begin{section}{Proof of the Main Reduction}

In this section, we prove our main reduction which states that if a class of distributions admits a locally small cover and is list decodable, then its \emph{mixture class} can be learned privately. Let us first state the theorem again.

\mainreduction*

The high level idea of the proof is based on a connection to the private common member selection problem.
To see this, assume class $\cF$ is list decodable and admits a locally small cover. Then we show that given some samples from any $f^*\in$ $\kmix(\cF)$, one can generate $T$ lists of dense mixtures (i.e., member of $(k,\frac{\alpha}{k})$-dense-mix($\cF$)) 
such that they have a common member (w.r.t.~$\kappa_{mix}$). However, there could be multiple common members w.r.t.~$\kappa_{mix}$ that are far from $f^*$ w.r.t.~$\dtv$. Therefore, we filter out all distributions that have a bad $\dtv$ distance with the original distribution $f^*$ (This can be done using Minimum Distance Estimator). Afterwards, we are guaranteed that all common members are also good w.r.t.~$\dtv$. Also, note that changing a data point can change at most one of the lists. Therefore, by using the private common member selector, we can choose a common member from these filtered lists while maintaining privacy. The formal proof is given below.

\begin{proof}
 Let $\alpha',\beta'\in(0,1)$. If $\cF$ is $\left(L, m,\alpha',\beta',0\right)$-list-decodable w.r.t.~$\dtv$ then using Lemma~\ref{lemma:list-decode-inc}, we get that $\cF$ is $(L_1, m_1,\alpha',2\beta',1-\alpha/k)$-list-decodable w.r.t.~$\dtv$, where $L_1=L\cdot (\frac{10ek\log(1/\beta')}{\alpha'})^m$, and $m_1 = \frac{2mk+8k\log(1/\beta')}{\alpha'}$.
   
   Let $f^*\in$ $\kmix(\cF)$ be the true distribution. Using Lemma~\ref{lemma:dense}, we can write $f^*=\gamma h + (1-~\gamma)f $, where $f \in (k,\frac{\alpha'}{k})$-dense-mix($\cF$), and $\gamma \in [0,\alpha')$. 
   Let $L_2=(\frac{kL_1}{\alpha'})^{k+1}\cdot (\frac{10e\log(1/2k\beta')}{1-\alpha'})^{m_1}$, and $m_2 =~ \frac{2m_1+8\log(1/2k\beta')}{1-\alpha'}$. 
   By Theorem~\ref{thm:list-mix}, we know that $(k,\frac{\alpha'}{k})$-dense-mix($\cF$) is $\left(L_2, m_2,\alpha',4k\beta',\alpha'\right)$-list-decodable w.r.t.~$\kappa_{mix}$. 

    Let $t_1 = k!(tk/\alpha')^k$, and $T=O\left(\frac{\log(1/\delta)+\log(t_1L_2/\beta')}{\eps}\right)$. For $i\in [T]$, draw $T$ disjoint datasets each of size $m_3=O\left(\frac{\log(L_2)+\log(1/\beta')}{\alpha'^2}\right) + m_2$. For each dataset, run the list decoding algorithm using $m_2$ samples from that dataset, and let $\cL_i$ denote the outputted list.

    As mentioned above, we know that $\densemix{(k, \frac{\alpha'}{k})}(\cF)$ is $(L_2, m_2, \alpha', 4k\beta', \alpha')$-list-decodable w.r.t.~$\kappamix$. Thus,
    for each $i\in [T]$, with probability at least $1-4k\beta'$, we have $\kappa_{mix}(f,\cL_i) \leq \alpha'$. 
    To convert this bound back to total variation distance, we use Lemma~\ref{lemma:kmix-to-tv} to get that $\dtv(f,\cL_i)\leq 3\alpha'/2$.

    Note that the size of each $\cL_i$ is at most $L_2$. By making use of the Minimum Distance Estimator (Theorem~\ref{thm:mde}), we can use the other $O(\frac{\log(L_2)+\log(1/\beta')}{\alpha'^2})$ samples from each datasets to find $\hat{f}_i \in \cL_i$ such that with probability at least $1-\beta'$ we have $\dtv(\hat{f}_i,f)\leq 3\cdot\dtv(f,\cL_i) + \alpha' \leq 11\alpha'/2$. 

    We then proceed with a filtering step. For each $i \in [T]$, we define $\cL_i' = \{ f' \in \cL_i\,:\, \dtv(f', \hat{f}_i) < 11\alpha'/2 \}$ to be the elements in $\cL_i$ that are close to $\hat{f}_i$.
 
    Using a union bound and a triangle inequality, we get that with probability at least $1-(4k+1)\beta'$ we have $\max_{f' \in \cL_i'} \dtv(f,f')\leq 11\alpha'$.
    
    Applying a union bound over all $T$ datasets, we conclude that with probability at least $1-(4k+1)\beta' T$, $f$ is a $(\alpha',1)$-common-member for $D=\{\cL_1',\cdots,\cL_T'\}$ w.r.t.~$\kappa_{mix}$, and $\max_{f'\in \cL_i'}\dtv(f,f')\leq 11\alpha'$ for all $i\in[T]$.

   The fact that $\cF$ admits a $(t,2\alpha')$-locally small $\alpha'$-cover, along with Lemma~\ref{lemma:densemix-cover}, implies that there exists an $(t_1,2\alpha')$-locally small $\alpha'$-cover $\cC$ for $\densemix{(k,\frac{\alpha'}{k})}(\cF)$.

     Note that $|\cL_i'| \leq |\cL_i| \leq L_2$. Now, we run the private common member selector (Algorithm~\ref{alg:pcms}) on $(D,\kappa_{mix}, \cC)$ to obtain  $\hat{f}$. Using Theorem~\ref{thm:pcms} and a union bound, we get that with probability at least $1-((4k+1)T+1)\beta'$, $\hat{f}$ is a $(2\alpha',0.9)$-common-member w.r.t.~$\kappa_{mix}$. 
    Therefore, Lemma~\ref{lemma:kmix-to-tv} implies that $\hat{f}$ is a $(3\alpha',0.9)$-common-member w.r.t.~$\dtv$. 

    Using the fact that for every $i\in [T]$, $\max_{f'\in \cL_i'} \dtv(f',f) \leq 11\alpha'$, we get that $\dtv(\hat{f},f)\leq 14 \alpha'$. Finally, triangle inequality implies that $\dtv(\hat{f},f^*) \leq \dtv(\hat{f},f)+\dtv(f,f^*) \leq 15 \alpha'$ with probability at least $1-((4k+1)T+1)\beta' \geq 1-6kT\beta'$. The total sample complexity is:
    \begin{align*}
       & T\cdot m_3 =  O\left(\frac{\log(1/\delta)+\log(t_1L_2/\beta')}{\eps}\cdot \left(\frac{\log(L_2)+\log(1/\beta')}{\alpha'^2}+m_2\right)\right) \\    &=\tilde{O}\left(\left(\frac{\log(1/\delta)+k\log(tL)}{\eps}+\frac{mk+k\log(1/\beta')}{\alpha'\eps}\right)\cdot \left(\frac{k\log(L)}{\alpha'^2}+\frac{mk+k\log(1/\beta')}{\alpha'^3}\right)\right).
    \end{align*}
    Finally, we substitute $\alpha'=\alpha/15$ and $\beta'= \frac{\beta\eps}{12ek\log(6ekt_1L_2/\eps\beta\delta)} < 1$.
    Applying Claim~\ref{claim: failure-prob} with $c_1=~6k/\eps, c_2=~t_1L_2/\delta$, we get that $\dtv(\hat{f},f^*)\leq \alpha$ with probability at least $1-\beta$.
    Moreover the order of the sample complexity remains unchanged.
  The given approach is private since changing one data point alters only one of the $T$ datasets, and therefore affects at most one 
  $\cL_i'$. The privacy guarantee then follows from Theorem~\ref{thm:pcms}.
\end{proof}

\end{section}

\begin{section}{Privately Learning GMMs}
Prior to this work, there was no a finite sample complexity upper bound for privately learning mixtures of unbounded Gaussians.
As an application of our general framework, we study the problem of privately learning GMMs.
This section provides the necessary ingredients for using our proposed framework. First, we show that the class of Gaussians is list decodable using sample compression schemes \citep{ashtiani2018nearly}. Second, we show that this class admits a locally small cover.  Putting together, we prove the first sample complexity upper bound for privately learning GMMs.

\subsection{List-decoding Gaussians using compression}
\label{subsec:compression}
As we stated in Remark~\ref{remark:gmm-pac-instead}, it is possible to use a PAC learner as a naive list decoding algorithm that outputs \textit{a single} Gaussian. However, doing so, results in a poor sample complexity for privately learning GMMs.
In this section, we provide a carefully designed list decoding algorithm for the class of Gaussians which results in a much better sample complexity for privately learning GMMs due to its mild dependence on the accuracy parameter ($1/\alpha$).

We reduce the problem of list decoding a class of Gaussians to the problem of compressing this class. Next, we use the result of \citet{ashtiani2018nearly} that the class of Gaussians is compressible. Finally, we conclude that this class is list decodable.  

\begin{remark}
    Our reduction is fairly general and works for any class of distributions. Informally, if a class of distributions is compressible, then it is list decodable. However, for the sake of simplicity we state this result only for the class of Gaussians in Lemma~\ref{lemma:list-decode-gaussians}.
\end{remark}

The method of sample compression schemes introduced by \citet{ashtiani2018nearly} is used for distribution learning. At a high level, given $m$ samples from a distribution and $t$ additional bits, if there exists an algorithm (i.e., decoder) that can approximately recover the original distribution given a small subset of (i.e., $\tau$ many) samples and $t$ bits, then one can create a list of all possible combinations of choosing $\tau$ samples and $t$ bits. Then one can pick a ``good'' distribution from the generated list of candidates using Minimum Distance Estimator \citep{yatracos1985rates}. 
Below, we provide the formal definition of sample compression schemes for learning distributions.

\begin{definition}[Compression schemes for distributions \citep{ashtiani2018nearly}] \label{def:compression}
     For a set of functions $\tau,m,t:(0,1)\rightarrow \bZ_{\geq 0}$, the class of distributions $\cF$ is said to be $(\tau,t,m)$-compressible if there exists a decoder $\cD$ such that for any $f \in \cF$ and any $\alpha \in (0,1)$, after receiving an i.i.d sample $\cS$ of size $m(\alpha)\log(1/\beta)$ from $f$, with probability at least $1-\beta$, there exists a sequence L of at most $\tau(\alpha)$ members of $\cS$ and a sequence $B$ of at most $t(\alpha)$ bits, such that $\dtv(D(L,B),f)\leq \alpha$.
\end{definition}

The following result states that the class of Gaussians is compressible.
\begin{lemma}[Lemma 4.1 of \cite{ashtiani2018nearly}] \label{lemma:gaussian-compression}
    Let $\alpha\in(0,1)$. The class $\cG$ is $(O(d),\tilde{O}(d^2\log(1/\alpha)),O(d))$-compressible.
\end{lemma}

Finally, the next lemma 
uses the above result to show that Gaussians are list decodable. 

\begin{lemma}\label{lemma:list-decode-gaussians}
    For any $\alpha,\beta \in (0,1)$, the class $\cG$ is $(L,m,\alpha,\beta,0)$-list-decodable w.r.t.~$\dtv$ for $L=~(d\log(1/\beta))^{\tilde{O}(d^2\log(1/\alpha))}$, and $m=O\left(d\log(1/\beta)\right)$.
\end{lemma}

An important feature of the above lemma (that is inherited from the compression result) is the mild dependence of $m$ and $L$ on $1/\alpha$: $m$ does not depend on $1/\alpha$ and $L$ has a mild polynomial dependence on it.
\begin{proof}
    We will prove that if a class of distributions $\cF$ is $(\tau,t,s)$-compressible, then it is \\$\left(O((s \log(1/\beta))^{t+\tau}),s\log(1/\beta),\alpha,\beta,0\right)$-list-decodable w.r.t.~$\dtv$. 
    Let $f\in \cF$, and $\cS$ be a set of $s\log(1/\beta)$ i.i.d.~samples drawn from $f$.
    Now using the decoder $D$, construct the list $\cL =\{D(L,B): L\subseteq \cS , |L|\leq \tau, B\in \{0,1\}^{t}\}$.
    Using the Definition~\ref{def:compression}, with probability at least $1-\beta$ there exists $\hat{f}\in \cL$ such that $\dtv(\hat{f},f)\leq \alpha$. Note that $|\cL|= O((s\log(1/\beta))^{(t+\tau)})$. Putting together with Lemma~\ref{lemma:gaussian-compression} concludes the result. 
\end{proof}

\subsection{A locally small cover for Gaussians}
In this section, we construct a locally small cover for the class of Gaussians using the techniques of \citet{aden2021sample}.
Explicitly constructing a locally small cover for Gaussians is a challenging task due to the complex geometry of this class.
Previously, \citet{aden2021sample} constructed locally small covers for location Gaussians (Gaussians with $I_d$ covariance matrix) and scale Gaussians (zero mean Gaussians). One might think by taking the product of these two covers we can simply construct a locally small cover for unbounded Gaussians (see Definition~\ref{def:general_gaussian}). However, the product of these two covers is not a valid cover for unbounded Gaussians as the $\dtv$ between two Gaussians can be very large even if their means are close to each other in euclidean distance.

To resolve this issue, we take a small cover for a small $\dtv$-ball of location gaussians around $N(0,I_d)$ and scale it using a small cover for the $\dtv$-ball of scale gaussians around $N(0,I_d)$. Showing that this is a valid and small cover for a small $\dtv$-ball of unbounded Gaussians around $N(0,I_d)$ is a delicate matter. To argue that this is a valid cover, we use the upper bound of \citet{devroye2018total} for the $\dtv$ distance between high-dimensional Gaussians. Showing that this is a small cover requires the recently proved lower bound on the $\dtv$ distance of high-dimensional Gaussians \citep{arbas2023polynomial}.

Once we created a small cover for a $\dtv$-ball of unbounded Gaussians around $N(0,I_d)$, we can transform it to a small cover for a  $\dtv$-ball of unbounded Gaussians around an arbitrary $N(\mu,\Sigma)$. Finally, we use Lemma~\ref{lemma:extend-cover} to show that there exists a global locally small cover for the whole space of Guassians.

\begin{definition} [Location Gaussians]
    Let $\cG^L=\{\cN(\mu,I_d): \mu \in \mathbb{R}^d\}$ be the class of d-dimensional location Gaussians.
\end{definition}
\begin{definition} [Scale Gaussians]
    Let $\cG^S=\{\cN(0,\Sigma): \Sigma \in S^d\}$ be the class of d-dimensional scale Gaussians.
\end{definition}

The next two lemmas propose small covers for location and scale Gaussians near $N(0,I)$. Recall that $B_{\dtv}(r,f,\cF)$ stands for the $\dtv$-ball of radius $r$ around $f$ (See Definition~\ref{def:ball}).
\begin{lemma} [Lemma 30 of \cite{aden2021sample}]\label{lemma:loc-cover}
     For any $0 < \alpha < \gamma <c$, where $c$ is a universal constant,  there exists an $\alpha$-cover $C^L$ for the set of distributions $B_{\dtv}(\gamma,N(0,I),\cG^L)$ of size $(\dfrac{\gamma}{\alpha})^{O(d)}$ w.r.t.~$\dtv$.
\end{lemma}

\begin{lemma}[Corollary 33 of \cite{aden2021sample}]\label{lemma:scale-cover}
     For any $0 < \alpha < \gamma <c$, and $\Sigma \in S^d$, where $c$ is a universal constant, there exists an $\alpha$-cover $C^S$ for the set of distributions $B_{\dtv}(\gamma,N(0,\Sigma),\cG^S)$ of size $(\dfrac{\gamma}{\alpha})^{O(d^2)}$ w.r.t.~$\dtv$.
\end{lemma}

The next theorem provides upper and lower bounds for $\dtv$ between two Gaussians, which will be used for constructing a small cover for unbounded Guassians near $N(0,I)$. 
\begin{theorem}[Theorem 1.8 of \cite{arbas2023polynomial}]  \label{thm:tv-bounds}
    Let $N(\mu_1,\Sigma_1),N(\mu_2,\Sigma_2)\in \cG$, and $\Delta = \max\{||\Sigma_1^{-1/2}\Sigma_2\Sigma_1^{-1/2}-I_d||_F, ||\Sigma_1^{-1/2}(\mu_1-\mu_2)||_2\}$. Then
    \begin{align*}
    \dtv(N(\mu_1,\Sigma_1),N(\mu_2,\Sigma_2)) \leq \frac{1}{\sqrt{2}}\Delta.
    \end{align*}
    Also, if $\dtv(N(\mu_1,\Sigma_1),N(\mu_2,\Sigma_2))\leq \frac{1}{600}$, we have:
    \begin{align*}
        \frac{1}{200}\Delta \leq \dtv(N(\mu_1,\Sigma_1),N(\mu_2,\Sigma_2)).
    \end{align*}
\end{theorem}

The next lemma proposes a small cover for the unbounded Guassians near $N(\mu,\Sigma)$ for any given $\mu$ and $\Sigma$. To do so, we combine the small covers from Lemma~\ref{lemma:loc-cover} and Lemma~\ref{lemma:scale-cover} in a way that it approximates any Gaussian near $N(0,I)$.
\begin{lemma}\label{lemma:cover-tv-ball}
    Let $0 < \alpha < \gamma \leq \frac{1}{600}$, $\mu \in \bR^d$, and $\Sigma \in S^d$. There exists an $\alpha$-cover for $B_{\dtv}(\gamma, N(\mu,\Sigma), \cG)$ of size at most $(\dfrac{\gamma}{\alpha})^{O(d^2)}$.
\end{lemma}
\begin{proof}
First, we construct a cover for $B_{\dtv}(\gamma, N(0,I), \cG)$, then we extend it to a cover for $B_{\dtv}(\gamma, N(\mu,\Sigma), \cG)$ using a linear transformation.

Let $\gamma \in(\alpha,\frac{1}{600})$ and consider the ball $\cB:=B_{\dtv}(\gamma,N(0,I),\cG)$.
      Let $\gamma_1 = 200\gamma$ and $C^L$ be an $\frac{\sqrt{2}}{200}\alpha$-cover for $B_{\dtv}(\gamma_1,N(0,I),\cG^L)$ from Lemma~\ref{lemma:loc-cover}. Also, let $\gamma_2 =200\gamma$ and $C^S$ be an $\frac{\sqrt{2}}{200}\alpha$-cover for $B_{\dtv}(\gamma_2,N(0,I),\cG^S)$ from Lemma~\ref{lemma:scale-cover}.
     
     We claim that $C^B=\{N(\Sigma^{1/2}\mu,\Sigma): N(\mu,I) \in C^L, N(0,\Sigma)\in C^S\}$ is an $\alpha$-cover for $\cB$.
     Let $N(\hat{\mu},\hat{\Sigma}) \in \cB$ so that
     $\dtv(N(0,I),N(\hat{\mu},\hat{\Sigma})) \leq \gamma \leq \frac{1}{600}$. Applying the lower bound of Theorem~\ref{thm:tv-bounds} with $\Sigma_1=I, \Sigma_2=\hat{\Sigma}, \mu_1=0,\mu_2=\hat{\mu}$ gives that
    \begin{align}
       & ||\hat{\Sigma}-I||_F \leq 200 \dtv(N(0,I),N(\hat{\mu},\hat{\Sigma})) \leq 200\gamma \quad \text{and} \\
       & ||\hat{\mu}||_2 \leq 200 \dtv(N(0,I),N(\hat{\mu},\hat{\Sigma})) \leq 200\gamma. \label{proof:norm2}
    \end{align}
    Moreover, applying Theorem~\ref{thm:tv-bounds} with $\Sigma_1=\hat{\Sigma}, \Sigma_2=I, \mu_1=\hat{\mu},\mu_2=0$ gives that
    \begin{align*}
       & ||\hat{\Sigma}^{-1/2}\hat{\mu}||_2 \leq 200 \dtv(N(0,I),N(\hat{\mu},\hat{\Sigma})) \leq 200\gamma.
    \end{align*}
    Next, applying the upper bound of Theorem~\ref{thm:tv-bounds} with $\mu_1=\mu_2=0,\Sigma_1=I,\Sigma_2=\hat{\Sigma}$ gives $\dtv(N(0,I),N(0,\hat{\Sigma})) \leq \frac{1}{\sqrt{2}} ||\hat{\Sigma}-I||_F \leq \frac{200}{\sqrt{2}}\gamma < 200\gamma = \gamma_2$.
     Therefore $N(0,\hat{\Sigma})\in B_{\dtv}(\gamma_2,N(0,1),\cG^S)$. Recall that $C^S$ is an $\frac{\sqrt{2}}{200}\alpha$-cover for $B_{\dtv}(\gamma_2,N(0,I),\cG^S)$. Thus, there exists $N(0,\tilde{\Sigma}) \in C^S$ such that $\dtv(N(0,\hat{\Sigma}),N(0,\tilde{\Sigma}))\leq \frac{\sqrt{2}}{200}\alpha$. Using the 
     lower bound of Theorem \ref{thm:tv-bounds} with $\Sigma_1=\tilde{\Sigma}, \Sigma_2=\hat{\Sigma}, \mu_1=0,\mu_2=0$ results in 
     \begin{equation}
     \label{proof:forb-norm}
     ||\tilde{\Sigma}^{-1/2}\hat{\Sigma}\tilde{\Sigma}^{-1/2}-I||_2 \leq ||\tilde{\Sigma}^{-1/2}\hat{\Sigma}\tilde{\Sigma}^{-1/2}-I||_F\leq 200\dtv(N(0,\hat{\Sigma}),N(0,\tilde{\Sigma})) \leq \sqrt{2} \alpha.
     \end{equation}
     Therefore $||\tilde{\Sigma}^{-1/2}\hat{\Sigma}\tilde{\Sigma}^{-1/2}||_2 = ||(\tilde{\Sigma}^{-1/2}\hat{\Sigma}^{1/2})(\tilde{\Sigma}^{-1/2}\hat{\Sigma}^{1/2})^T||_2 \leq 1+\sqrt{2}\alpha$. Finally, we get $||\tilde{\Sigma}^{-1/2}\hat{\Sigma}^{1/2}||_2 \leq \sqrt{1+\sqrt{2}\alpha} \leq \sqrt{2}$.

     Now let $\hat{v}=\hat{\Sigma}^{-1/2}\hat{\mu}$. From \ref{proof:norm2} we know that $||\hat{v}||_2\leq 200\gamma$.
     Therefore we have $||\tilde{\Sigma}^{-1/2}\hat{\mu}||_2 = ||\tilde{\Sigma}^{-1/2}\hat{\Sigma}^{1/2}\hat{v}||_2 \leq ||\tilde{\Sigma}^{-1/2}\hat{\Sigma}^{1/2}||_2 ||\hat{v}||_2 \leq 200\gamma\sqrt{2}$.
    Using the upper bound of Theorem~\ref{thm:tv-bounds} with $\Sigma_1=I$, $\Sigma_2=I$, $\mu_1=\tilde{\Sigma}^{-1/2}\hat{\mu}$, and $\mu_2=0$
    gives $\dtv(N(\tilde{\Sigma}^{-1/2}\hat{\mu},I),N(0,I))\leq \frac{1}{\sqrt{2}}||\tilde{\Sigma}^{-1/2}\hat{\mu}||_2 \leq \frac{200\gamma\sqrt{2}}{\sqrt{2}}= \gamma_1$.
    Thus $N(\tilde{\Sigma}^{-1/2}\hat{\mu},I) \in B_{\dtv}(\gamma_1,N(0,1),\cG^L)$. Recall that $C^L$ is an $\frac{\sqrt{2}}{200}\alpha$-cover for $B_{\dtv}(\gamma_1,N(0,I),\cG^L)$. Therefore, there exists $N(\tilde{\mu},I)\in C^L$ such that $\dtv(N(\tilde{\mu},I),N(\tilde{\Sigma}^{-1/2}\hat{\mu},I))\leq \frac{\sqrt{2}}{200}\alpha$. Using the
    lower bound of Theorem~\ref{thm:tv-bounds} with 
    $\Sigma_1=I, \Sigma_2=I, \mu_1=\tilde{\Sigma}^{-1/2}\hat{\mu},\mu_2=\tilde{\mu}$, we can write $||\tilde{\Sigma}^{-1/2}\hat{\mu}-\tilde{\mu}||_2 = ||\tilde{\Sigma}^{-1/2}(\tilde{\Sigma}^{1/2}\tilde{\mu}-\hat{\mu})||_2 \leq \sqrt{2} \alpha$. Putting together with  $\ref{proof:forb-norm}$, 
    we can use the upper bound in Theorem~\ref{thm:tv-bounds} with $\Sigma_1=\tilde{\Sigma}$, $\Sigma_2=\hat{\Sigma}$, $\mu_1=\tilde{\Sigma}^{1/2}\tilde{\mu}$, and $\mu_2=\hat{\mu}$ to get that $\dtv(N(\tilde{\Sigma}^{1/2}\tilde{\mu},\tilde{\Sigma}),N(\hat{\mu},\hat{\Sigma})) \leq \alpha$. Note that $N(\tilde{\Sigma}^{1/2}\tilde{\mu},\tilde{\Sigma}) \in C^B$. Hence, $C^B$ is an $\alpha$-cover for $\cB$. Moreover, we have $|C^B|=|C^L||C^S|\leq (\frac{\gamma_1}{\alpha})^{O(d)}(\frac{\gamma_2}{\alpha})^{O(d^2)} = (\frac{\gamma}{\alpha})^{O(d^2)}$.

    Now, we propose a cover for $B_{\dtv}(\gamma,N(\mu,\Sigma),\cG)$. Note that using Lemma~\ref{lemma:tv-transorm}, for any $\Sigma, \Sigma_1, \Sigma_2 \in S^d$, we have:
    \begin{align*}
\dtv(N(0,\Sigma^{1/2}\Sigma_1\Sigma^{1/2}),N(0,\Sigma^{1/2}\Sigma_2\Sigma^{1/2})) = \dtv(N(0,\Sigma_1),N(0,\Sigma_2)).
    \end{align*}
    Note that equality holds since the mapping is bijection. Next, create the set $C^{B_\Sigma}=\{N(\mu,\Sigma^{1/2}\Sigma'\Sigma^{1/2}):N(\mu,\Sigma')\in C^B\}$. Note that $C^{B_\Sigma}$ is an $\alpha$-cover for $B_{\dtv}(\gamma,N(0,\Sigma),\cG)$ since the $\dtv$ distance between every two distributions in $C^B$ remains same (i.e. does not increase) after this transformation. Finally, the set $C^{B_{\mu,\Sigma}}=\{N(\mu+\mu',\Sigma'): N(\mu',\Sigma')\in C^{B_\Sigma}\}$ is the desired $\alpha$-cover for $B_{\dtv}(\gamma,N(\mu,\Sigma),\cG)$ since it is the shifted version of $C^{B_\Sigma}$. Also, we have $|C^{B_{\mu,\Sigma}}|=|C^{B_\Sigma}|=|C^B|\leq (\frac{\gamma}{\alpha})^{O(d^2)}$.
\end{proof}

The next lemma provides a useful tool for creating (global) locally small covers. Informally, given a class of distributions, if there exists a small cover for a small ball around each distribution in the class, then there exists a (global) locally small cover for the whole class. 
\begin{lemma}[Lemma 29 of \cite{aden2021sample}] \label{lemma:extend-cover}
    Consider a class of distributions $\cF$ and let $0< \alpha < \gamma < 1$. If for every $f \in \cF$ the $B_{\dtv}(\gamma,f,\cF)$ has an $\alpha$-cover of size no more than t, then there exists a $(t,\gamma)$-locally small $2\alpha$-cover for $\cF$ w.r.t.~$\dtv$.
\end{lemma}

The proof of the above lemma is non-constructive and uses Zorn's lemma. An immediate implication of Lemma~\ref{lemma:cover-tv-ball} and Lemma~\ref{lemma:extend-cover} is the existence of a locally small cover for unbounded Gaussians.

\begin{lemma}
    \label{cor:gaussians_locally_small_cover}
    For any $0< \alpha < \gamma \leq \frac{1}{600}$, there exists a $((2\gamma/\alpha)^{O(d^2)}, \gamma)$-locally small $\alpha$-cover for the class $\cG$ w.r.t.~$\dtv$.
\end{lemma}

\subsection{Learning GMMs}
In this subsection, we prove the first sample complexity upper bound for privately learning mixtures of unbounded Gaussians. We use the fact that Gaussians are \textit{(1)} list decodable, and \textit{(2)} admit a locally small cover. Putting this together with our main theorem for learning mixture distributions, we conclude that GMMs are privately learnable.

\maingmm*

\begin{proof}
    According to Lemma~\ref{lemma:list-decode-gaussians}, the class $\cG$ is $(L,m,\alpha,\beta,0)$-list-decodable w.r.t.~$\dtv$ for $L = (d\log(1/\beta))^{\tilde{O}(d^2\log(1/\alpha))}$ and $m=O\left(d\log(1/\beta)\right)$. Moreover, Lemma~\ref{cor:gaussians_locally_small_cover} implies that $\cG$ admits a $(t,2\alpha)$-locally small $\alpha$-cover, where $t=4^{O(d^2)}$. Using Theorem~\ref{thm:main}, we get that $\kmix(\cG)$ is $(\eps,\delta)$-DP $(\alpha,\beta)$-PAC-learnable using
    \begin{align*}
      &
\tilde{O}\left(\left(\frac{\log(1/\delta)+k\log(tL)}{\eps}+\frac{mk+k\log(1/\beta)}{\alpha\eps}\right)\cdot \left(\frac{k\log(L)}{\alpha^2}+\frac{mk+k\log(1/\beta)}{\alpha^3}\right)\right)\\
    &
=\tilde{O}\left(\left(\frac{\log(1/\delta)+kd^2}{\eps}+\frac{kd\log(1/\beta)}{\alpha\eps}\right)\cdot \left(\frac{kd^2}{\alpha^2}+\frac{kd\log(1/\beta)}{\alpha^3}\right)\right)\\
   &
 =\tilde{O}\left(\frac{kd^2\log(1/\delta)+k^2d^4}{\alpha^2\eps}+ \frac{kd\log(1/\delta)\log(1/\beta)+k^2d^3\log(1/\beta)}{\alpha^3\eps}+\frac{k^2d^2\log^2(1/\beta)}{\alpha^4\eps}\right) 
    \end{align*}
    samples.
\end{proof} 

\end{section}

\begin{section}{More on Related Work}
In this section, we go through some related works on private learning of Gaussians and their mixtures.
    \subsection{Privately learning Gaussians}
        \citet{karwa2018finite} proposed a sample and time efficient algorithm for estimating the mean and variance of bounded univariate Gaussians under pure DP. They also provide a method for the unbounded setting under approximate DP.
        Later, other methods for learning high-dimensional Gaussians with respect to total variation distance were introduced by \citet{kamath2019privately,biswas2020coinpress}. 
        They assume the parameters of the Gaussians have bounded ranges. As a result, the sample complexity of their method depends on the condition number of the covariance matrix, and the range of the mean.
        
         Afterwards, \citet{aden2021sample} proposed the first finite sample complexity upper bound for privately learning unbounded high-dimensional Gaussians, which was nearly tight and matching the lower bound of \citet{kamath2022new}.
         Their method was based on a privatized version of Minimum Distance Estimator \citep{yatracos1985rates} and inspired by the private hypothesis selection of \citet{bun2021private}. One downside of \citet{aden2021sample} is that their method is not computationally efficient. 

         There has been several recent results on computationally efficient learning of unbounded Gaussians~\citep{kamath2022private,kothari2022private,ashtiani2022private}, with the method of~\citet{ashtiani2022private} achieving a near-optimal sample complexity using a sample-and-aggregate-based technique. Another sample-and-aggregate framework that can be used for this task is FriendlyCore~\citep{tsfadia2022friendlycore}.
         The methods of \citet{ashtiani2022private,kothari2022private} also work in the robust setting achieving sub-optimal sample complexities.
         Recently, \citet{alabi2023privately} improved this result in terms of dependence on the dimension. Finally,  \citet{hopkins2023robustness} achieved a robust and efficient learner with near-optimal sample complexity for unbounded Gaussians.

         In the pure DP setting,  \citet{hopkins2022efficient} proposed a method for efficiently learning Gaussians with bounded parameters. There are some other related works on private mean estimation w.r.t.~Mahalanobis distance in \citet{brown2021covariance,duchi2023fast,brown2023fast}. 

        Several other related studies have explored the relationship between robust and private estimation, as seen in \citet{dwork2009differential,georgiev2022privacy,liu2022differential,hopkins2023robustness,asi2023robustness}. Also, there have been investigations into designing estimators that achieve both privacy and robustness at the same time \citep{liu2021robust}.

    \subsection{Parameter estimation for GMMs}
        In this setting, upon receiving i.i.d. samples from a GMM, the goal is to estimate the parameters of the mixture. In the non-private setting there is an extensive line of research for parameter learning of GMMs \citep{dasgupta1999learning,sanjeev2001learning,vempala2004spectral,achlioptas2005spectral,brubaker2008isotropic,kalai2010efficiently,belkin2009learning,hardt2014sharp,hsu2013learning,anderson2014more,regev2017learning,kothari2018robust,hopkins2018mixture,liu2022clustering,feldman2006pac,moitra2010settling,belkin2010polynomial,bakshi2022robustly,liu2021settling,liu2022learning}.

    Under the boundedness assumption there has been a line of work in privately learning parameters of GMMs \citep{nissim2007smooth,vempala2004spectral,chen2023private,kamath2019differentially,achlioptas2005spectral,cohen2021differentially}. The work of  \citet{bie2022private} approaches the same problem by taking the advantage of public data. Recently, \citet{arbas2023polynomial} proposed an efficient method for reducing the private parameter estimation of unbounded GMMs to its non-private counterpart.

    Note that in parameter estimation of GMMs (even in the non-private setting), the exponential dependence of the sample complexity on the number of components is inevitable \citep{moitra2010settling}.

      \subsection{Density estimation for GMMs}
    In density estimation, which is the main focus of this work, the goal is to find a distribution which is close to the underlying distribution w.r.t.~$\dtv$. Unlike parameter estimation, the sample complexity of density estimation can be polynomial in both the dimension and the number of components. In the non-private setting, there has been several results about the sample complexity of learning GMMs~\citep{devroye2001combinatorial, ashtiani2018sample}, culminating in the work of~\citep{ashtiani2018nearly,ashtiani2020near} which gives the near-optimal bound of~$\tilde{\Theta}(kd^2/\alpha^2)$.  

 There are some researches on designing computationally efficient learners for one-dimensional GMMs \citep{chan2014efficient,acharya2017sample,liu2022robust,wu2018improved,li2017robust}. However, for general GMMs it is hard to come up with a computationally efficient learner due to the known statistical query lower bounds \citep{diakonikolas2017statistical}.

 In the private setting, and under the assumption of bounded components, one can use the Private Hypothesis Selection of \citet{bun2021private} or private Minimum Distance Estimator of \citet{aden2021sample} to learn classes that admit a finite cover under the constraint of pure DP. \citet{bun2021private} also proposes a way to construct such finite cover for mixtures of finite classes.

 Later, \citet{aden2021privately} introduced the first polynomial sample complexity upper bound for learning unbounded axis-aligned GMMs under the constraint of approximate DP. They extended the idea of stable histograms used in \citet{karwa2018finite} to learn univariate GMMs. However, this idea can't be generalized to general GMMs, as it is not clear how to learn even a single high-dimensional Gaussian using stability-based histogram. 

     Another related work is the lower bound on the sample complexity of privately learning GMMs with known covariance matrices \citep{acharya2021differentially}.

    In an independent and concurrent work, \citet{ben2023private} proposed a pure DP method for learning general GMMs, assuming they have access to additional public samples. 
    In fact, they use sample compression (with public samples) to find a list of candidate GMMs. Since this list is created using public data, they can simply choose a good member of it using private hypothesis selection. We also create such lists in the process. However, most of the technical part of this paper is dedicated to guarantee privacy (with only access to private data). In fact, it is challenging to privatize compression-based approaches since by definition, they heavily rely on a few data points in the data set.  
    \citet{ben2023private} also study an equivalence between public-private learning, list decodable learning and sample compression schemes.

We also use sample compression in our list decoding algorithm for Gaussians. 
    
    Our work is the first polynomial sample complexity upper bound for privately learning mixtures of general GMMs without any restrictive assumptions. 
    Designing  a private and computationally efficient density estimator for GMMs remains an open problem even in the one-dimensional setting.
\end{section}

\ifthenelse{\boolean{usenatbib}}{
  % Use natbib package
  \bibliographystyle{plainnat}
  \bibliography{refs}
}{
  % Use biblatex package
  \printbibliography
}

\section{Appendix}
\begin{lemma}\label{lemma:tv-transorm}
    Let $f$ be an arbitrary function, and $X,Y$ be two random variables with the same support. Then $\dtv(f(X),f(Y)) \leq \dtv(X,Y)$. \\
\end{lemma}
\begin{proof}
    \begin{align*}
        \dtv(f(X),f(Y)) = \sup_{A\in \cX} \prob{f(X)\in A} - \prob{f(Y) \in A} = \sup_{A\in \cX} \prob{X\in f^{-1}(A)} - \prob{Y \in f^{-1}(A)} \leq \dtv(X,Y) 
    \end{align*}
\end{proof}

\begin{claim}
    \label{claim:failure_prob_inequality}
    Let $x \geq 1$.
    Then $1 + \frac{\log 2}{x} + \frac{\log x}{x} < 2$.
\end{claim}
\begin{proof}
    Let $f(x) = 1 + \frac{\log 2}{x} + \frac{\log x}{x}$.
    Then $f'(x) = -\frac{\log 2}{x^2} + \frac{1 - \log x}{x^2} = \frac{1 - \log(2x)}{x^2}$.
    Note that $f'(x)$ is decreasing so $f$ is concave.
    In addition, $x = e/2$ is the only root of $f'$ so $f$ is maximized at $e/2$.
    Thus, $f(x) \leq f(e/2) = 1 + \frac{2}{e} < 2$.
\end{proof}

\begin{claim} \label{claim: failure-prob}
    For $c_1,c_2,1/\beta \geq 1$, let $\beta' = \frac{\beta}{2ec_1 \log(ec_1c_2/\beta)}$, then $\beta' \leq \frac{\beta}{2ec_1} \leq \frac{1}{ec_1}$, and $c_1\beta' \log(c_2/\beta') \leq \beta$.
\end{claim}
\begin{proof}

\begin{align*}
    c_1\beta' \log(c_2/\beta')
    & = c_1 \frac{\beta}{2ec_1 \log(ec_1c_2/\beta)} \cdot \left[ \log(ec_1c_2/\beta) + \log(2) + \log\log(ec_1c_2/\beta) \right] \\
    & = \frac{\beta}{2e} \cdot \left[1 + \frac{\log 2}{\log(ec_1c_2/\beta)} + \frac{\log\log(ec_1c_2/\beta)}{\log(ec_1c_2/\beta)} \right] \\
    & \leq \beta/e.
\end{align*}
where in the last inequality, we used Claim~\ref{claim:failure_prob_inequality} with $x = \log(ec_1c_2/\beta) \geq 1$.
\end{proof}

\end{document}